\documentclass{article}
\usepackage{iclr2027_conference,times}
\iclrfinalcopy

\usepackage{amsmath,amsfonts,bm}

\def\eqref#1{equation~\ref{#1}}

\def\1{\bm{1}}

\DeclareMathAlphabet{\mathsfit}{\encodingdefault}{\sfdefault}{m}{sl}
\SetMathAlphabet{\mathsfit}{bold}{\encodingdefault}{\sfdefault}{bx}{n}

\DeclareMathOperator*{\argmin}{arg\,min}

\usepackage{amsmath,amsfonts,amssymb,mathtools,dsfont}
\usepackage[table]{xcolor}
\usepackage{wrapfig}
\usepackage{algpseudocode}
\usepackage{pifont}
\newcommand{\cmark}{\ding{51}}
\newcommand{\xmark}{\ding{55}}
\usepackage{booktabs,tabularx,array,graphicx}
\newcolumntype{L}[1]{>{\raggedright\arraybackslash}p{#1}}
\usepackage{hyperref}
\usepackage{enumerate}
\newtheorem{theorem}{Theorem}

\newtheorem{prop}{Proposition}

\newtheorem{lemma}{Lemma}
\newtheorem{definition}{Definition}
\newtheorem{remark}{Remark}
\newtheorem{assump}{Assumption}
\renewcommand{\eqref}[1]{\textup{(\ref{#1})}}

\newcommand{\cO}{\mathcal{O}}

\newcommand{\inter}[1]{\ifblank{#1}{\inter}{\textnormal{int}(#1)}}

\newcommand{\bR}{\mathbb{R}}

\newcommand{\cR}{\mathcal{R}}

\newcommand{\Rm}{R_\textnormal{max}}

\newcommand{\gapK}{\mathcal{G}_K}

\newcommand{\proofend}{\hfill$\square$}

\title{\raggedright Sample-Efficient Reinforcement Learning with General Utilities via Neural Dual Critics}

\author{Ankur Naskar\\
Indian Institute of Science, India
\\
\texttt{ankurnaskar@iisc.ac.in}
\And
Vaneet Aggarwal\\
Purdue University, USA\\
\texttt{vaneet@purdue.edu}
}

\begin{document}
\maketitle
\fancyhead{}
\renewcommand{\headrulewidth}{0pt}

\begin{abstract}
    Reinforcement learning with general utilities optimizes nonlinear functionals of policy-induced occupancy measures. We develop a natural actor-critic method for concave RLGU objectives using an exact variational dual representation, whose existence follows from concavity and differentiability, with explicit constructions for a broad class of separable occupancy utilities. Rather than estimating the occupancy measure and then forming the policy-gradient pseudo-reward, our method learns this pseudo-reward directly as a function of the state-action pair using a neural critic. The critic is trained by projected stochastic updates, while multilevel Monte Carlo combines outputs across training horizons to control its bias at logarithmic expected cost. Our analysis separates conditional actor bias from estimator second moments and accounts explicitly for dual optimization, critic approximation, and neural linearization errors. Under policy-transfer, dual-stability, and local critic-regularity conditions, the method achieves $\widetilde{\mathcal O}(T^{-1/2})$ global convergence above policy, critic-class, and finite-width approximation errors using $\widetilde{\mathcal O}(T)$ expected environment transitions and reference samples. For a fixed critic representation, this yields $\widetilde{\mathcal O}(\epsilon^{-2})$ sampling complexity above the approximation floor, covering infinite state-action spaces and occupancy utilities beyond fixed finite vectors of returns.
\end{abstract}

\section{Introduction}
\label{s: introduction}

Standard reinforcement learning (RL) maximizes an expected discounted sum of rewards, an objective linear in the policy-induced occupancy measure. Many applications, however, depend on the occupancy distribution itself rather than on a fixed linear reward. For example, imitation learning seeks to match the learner's occupancy measure to an expert's target distribution, while exploration may encourage high visitation entropy \citep{ho2016generative,hazan2019provably}. Reinforcement learning with general utilities (RLGU) provides a unified formulation:

$$
    J(\theta)=F(\lambda_\theta),
$$

where $\lambda_\theta$ is the normalized discounted state-action occupancy measure and $F$ is a utility functional. We focus on concave utilities and generally parameterized policies. Although $J$ is generally nonconcave in $\theta$, concavity of $F$ in the occupancy measure provides the global comparison structure used in our analysis \citep{zhang2020variational,barakat2025global}.

The main statistical challenge is the policy-dependent pseudo-reward $r_\theta=\nabla_\lambda F(\lambda_\theta)$. Holding this reward fixed, the policy gradient takes the form
\begin{equation}\label{e: intro.gradient}
\nabla_\theta J(\theta) =(1-\gamma)\left.\nabla_\vartheta V^{\pi_\vartheta}(r_\theta)\right|_{\vartheta=\theta}.
\end{equation}
For small finite spaces, one can estimate the occupancy measure directly and then evaluate $r_\theta$. In larger spaces, however, occupancy approximation becomes an additional learning problem. Existing linear and parametric occupancy estimators yield $\widetilde{\mathcal O}(\epsilon^{-4})$ guarantees for stationarity \citep{barakat2023reinforcement} and global optimality \citep{barakat2026global}, respectively.

In this work, we study concave RLGU objectives through an exact variational dual representation, whose existence follows from our utility regularity assumptions. We give explicit dual constructions for a broad class of separable occupancy utilities. Instead of first estimating $\lambda_\theta$ and then forming $r_\theta$, we directly learn the dual pseudo-reward as a function of $(s,a)$ using a neural critic. This places the approximation requirement on the reward function actually needed by the actor, rather than on the full occupancy density, and allows the framework to extend naturally to infinite state-action spaces.

Our method and analysis hinge on two technical ideas.

\paragraph{Technical novelty: learning the dual reward.}
Building on Fenchel duality and variational policy gradients \citep{zhang2020variational,zahavy2021reward}, we learn the pseudo-reward through stochastic neural dual optimization (Section~\ref{s: dual.representation}). Our finite-time analysis uses dual curvature and feature conditioning to establish a stable recursion toward the constrained dual minimizer, separating dual-class approximation from neural linearization error. Table~\ref{tab: utility.examples} in Appendix~\ref{appendix: example.separable} gives reference-sampling estimators for the required dual gradients.

\paragraph{Technical novelty: controlling the bias of a trained critic.}
The critic's $\mathcal O(H^{-1/2})$ root-mean-square error above its approximation floor would require long training if used directly as actor bias. Adapting MLMC and neural actor-critic analyses \citep{blanchet2015unbiased,ganesh2025order,naskar2026hierarchical}, we apply randomized corrections to critic output functions at dyadic horizons, matching the longest run's mean at logarithmic expected training cost (Lemma~\ref{thm: MLMC.critic.bound}). Here MLMC controls a learned dual reward rather than finite scalarization coefficients. Separating actor bias from second moments in the KL-potential bound (Lemma~\ref{lem: opt.gap}) and adjusting the actor stepsize preserves the $\widetilde{\cO}(T^{-1/2})$ statistical rate with explicit finite-width error.

\paragraph{Contributions:} Our key contributions are as follows:

\begin{itemize}
    \item We develop a neural dual approach to concave RLGU that learns the utility-induced pseudo-reward directly, avoiding explicit occupancy estimation. For separable utilities, we derive explicit dual objectives and sampling formulas (Table~\ref{tab: utility.examples}, and Proposition~\ref{prop: separable.dual} in Appendix~\ref{appendix: example.separable}).

    \item We introduce a function-valued MLMC correction across critic-training horizons that inherits the bias of a long critic run at logarithmic expected training cost. We establish critic bounds and a natural-gradient analysis separating conditional actor bias from estimator second moments (Lemmas~\ref{thm: vanilla.critic.bound}, \ref{thm: MLMC.critic.bound}, and \ref{thm: NPG.bound}).

    \item Combining these ingredients, we obtain $\widetilde{\mathcal O}(\epsilon^{-2})$ expected sampling complexity above an explicit approximation floor for a fixed critic representation, including infinite state--action spaces (Theorem~\ref{thm: main.result}). Table~\ref{tab:comparison} contrasts this guarantee with existing RLGU results in representation, scope, and complexity.
\end{itemize}

\begin{table*}[t]
\centering

\caption{Comparison of representations, theoretical support for infinite state--action spaces, and sample complexity in RLGU.
Here,~$\widetilde{\mathcal O}$ suppresses logarithmic factors. Our bound holds above the approximation floor for a fixed critic representation under the stated policy-transfer, dual-stability, and local neural-regularity assumptions.~$^\dagger$ Assumes an oracle cost player whose implementation cost is excluded.~$^\ddagger$ Establishes first-order stationarity; all other entries concern global optimality.}
\label{tab:comparison}

\small
\setlength{\tabcolsep}{4pt}
\renewcommand{\arraystretch}{1.1}

\begin{tabular}{l c c c}
\toprule
\textbf{Reference}
& \textbf{Representation}
& \textbf{Infinite Spaces}
& \textbf{Sample Complexity} \\
\midrule

\cite{zhang2021convergence}
& Tabular
& \xmark
& $\widetilde{\mathcal{O}}(\epsilon^{-2})$ \\

\cite{zahavy2021reward}
& Tabular (UCRL2)
& \xmark
& $\widetilde{\mathcal{O}}(\epsilon^{-2})^{\dagger}$ \\

\cite{barakat2023reinforcement}
& Linear occupancy
& \xmark
& $\widetilde{\mathcal{O}}(\epsilon^{-4})^{\ddagger}$ \\

\cite{barakat2026global}
& Parametric occupancy
& \xmark
& $\widetilde{\mathcal{O}}(\epsilon^{-4})$ \\

\midrule
\rowcolor{gray!10}
\textbf{This work}
& \textbf{Neural dual critic}
& \cmark
& $\widetilde{\mathcal{O}}(\epsilon^{-2})$ \\
\bottomrule
\end{tabular}

\end{table*}

\subsection{Related Works:}

\paragraph{General-utility and convex reinforcement learning:}
\citet{hazan2019provably} study concave visitation objectives through planning, and \citet{zhang2020variational} develop the variational policy-gradient and Fenchel-dual formulation. \citet{zahavy2021reward} express convex RL as a policy-cost game, while \citet{geist} connect it to mean-field games and Frank-Wolfe updates. We build on these objective and duality foundations, but analyze a trained neural pseudo-reward.

\paragraph{Large-state RLGU and occupancy approximation:}
\citet{barakat2023reinforcement} provide variance-reduced tabular methods and a linear occupancy approximation method. \citet{barakat2025global} obtain global guarantees using maximum-likelihood occupancy estimation. Table~\ref{tab:comparison} distinguishes their representation assumptions and guarantees. Our dual-reward approach instead imposes approximation and stability conditions on the utility gradient rather than the occupancy density.

\paragraph{Concave utility RL.}
Finite-objective concave utility RL (CURL) optimizes a concave function of finitely many expected returns \citep{agarwal2023joint,agarwal2022curl,bai2022joint}. For finite state-action spaces, it is equivalent to concave RLGU: $M=|\mathcal S||\mathcal A|$ indicator rewards recover the occupancy vector through normalized returns \citep{ganesh2026breakingbiasbarrierconcave}. In infinite spaces, general RLGU is strictly more expressive, since finitely many occupancy integrals cannot represent every occupancy utility. Proposition~\ref{prop: finite.curl} in Appendix~\ref{appendix: CURL} gives the finite reduction and a continuous-state counterexample.

\paragraph{Variational and Fenchel-dual approaches:}
Our stochastic dual learning also relates to variational divergence estimation \citep{nowozin2016fgan}. The actor analysis adapts approximate NPG arguments from \citet{agarwal2021theory,ganesh2025order,ganesh2025sharper}, while the critic-output correction adapts MLMC ideas from \citet{blanchet2015unbiased,naskar2026hierarchical}. \citet{ganesh2026breakingbiasbarrierconcave} obtain the $\epsilon^{-2}$ exponent for finite-objective concave scalarization; our extension treats a function-valued dual reward and accounts for its optimization, approximation, and sampling errors.

\section{Problem Setup}
\label{s: problem.setup}

Let $\mathcal M=(\mathcal S,\mathcal A,\mathcal P,\rho,\gamma)$ be a discounted MDP on compact metric state and action spaces with their Borel sigma-algebras. Here $\mathcal P$ is a transition kernel, $\rho$ an initial-state distribution, and $\gamma\in(0,1)$. Finite spaces are included. A stationary policy $\pi$ induces $s_0\sim\rho$, $a_t\sim\pi(\cdot\mid s_t)$, and $s_{t+1}\sim\mathcal P(\cdot\mid s_t,a_t)$. Its normalized occupancy measure is
\begin{equation}\label{e: occupancy}
    \lambda_\pi(B)=(1-\gamma)\sum_{t=0}^\infty\gamma^t\mathbb P_\pi\big((s_t,a_t)\in B\big)
\end{equation}
for measurable $B\subseteq\mathcal S\times\mathcal A$, where $\mathbb P_\pi$ denotes the trajectory law induced by the policy $\pi$.

We consider differentiable policy densities $\pi_\theta(a\mid s)$ with respect to a fixed action reference measure and parameter-independent support, with $\theta\in\mathbb R^d$. The policy assumptions stated below hold along all update segments and at the comparison policy. Write $\lambda_\theta=\lambda_{\pi_\theta}$ and $d_\theta$ for its state marginal. For a convex domain $\mathcal D$ of probability measures containing all policy occupancies, our objective is
\begin{equation}\label{e: goal}
    F^*=F(\lambda_{\theta^*})=\max_{\theta\in\mathbb R^d}F(\lambda_\theta).
\end{equation}
We assume the maximum is attained; the same analysis applies to any fixed comparator with finite initial KL divergence. We now state the regularity required of the utility functional.
\begin{assump}[Concavity and Differentiability]\label{a: utility}
    $F:\mathcal D\to\mathbb R$ is concave and differentiable at each policy occupancy. Its functional derivative $r_\theta=\nabla_\lambda F(\lambda_\theta)$ represents directional derivatives along segments in $\mathcal D$, and the chain rule holds along differentiable policy paths. Moreover, $|r_\theta|_\infty\leq \Rm$ uniformly over the policies under consideration.
\end{assump}
Concavity provides the global comparison used in the convergence analysis, while bounded pseudo-rewards control the moments of the policy-gradient estimator. 

\subsection{Natural Policy Gradient}
\label{s: NPG}

We next express the utility gradient in a form suitable for natural policy optimization. Let $u_\theta(s,a)=\nabla_\theta\log\pi_\theta(a\mid s)$ denote the score. For a fixed bounded reward $r$, define
\begin{equation}\label{e: value}
    V^{\pi_\theta}(r) = \mathbb E_{\pi_\theta}\left[\sum_{t\geq0}\gamma^t r(s_t,a_t)\right]=\frac{1}{1-\gamma}\int rd\lambda_\theta.
\end{equation}
Applying the functional chain rule to $J(\theta)=F(\lambda_\theta)$ and then the likelihood-ratio identity gives
\begin{align}\label{e: policy.gradient}
    \nabla_\theta J(\theta)&=(1-\gamma)\left.\nabla_\vartheta V^{\pi_\vartheta}(r_\theta)\right|_{\vartheta=\theta}
    =(1-\gamma)\mathbb E_{\pi_\theta}\left[
    \sum_{t\geq0}\gamma^t r_\theta(s_t,a_t)
    \sum_{j=0}^t u_\theta(s_j,a_j)\right].
\end{align}
Here $r_\theta$ is held fixed when differentiating with respect to $\vartheta$. This is the general-utility policy-gradient identity \citep{zhang2020variational,kumar2024policy}; bounded rewards and scores justify the interchange with the discounted sum.

Natural policy gradient (NPG) preconditions this gradient using the Fisher matrix \citep{kakade2001natural}:
\begin{equation}\label{e: fisher}
    \mathcal I(\theta)= \mathbb E_{\lambda_\theta}[u_\theta u_\theta^\top],
    \qquad w^*(\theta) = \mathcal I(\theta)^{-1}\nabla_\theta J(\theta),
\end{equation}
where to find the NPG vector $w^*(\theta),$ we minimize the quadratic objective
\begin{equation}\label{e: compatibility.func}
    \ell(\theta,w) = \frac12 w^\top\mathcal I(\theta)w - w^\top\nabla_\theta J(\theta).
\end{equation}

The policy-gradient theorem~\citep{sutton1999policy} gives
$\nabla_\theta J(\theta)=\mathbb E_{\lambda_\theta}[u_\theta A_\theta^{(r_\theta)}]$,
where $A_\theta^{(r)}$ is the advantage function associated with a fixed reward $r$. Specifically, define
\begin{equation}\label{e: advantage}
    \begin{aligned}
        Q_\theta^{(r)}(s,a) &=\mathbb E_{\pi_\theta}\left[\sum_{t\geq0}\gamma^t r(s_t,a_t) \mid s_0=s,a_0=a\right],
        \\
        A_\theta^{(r)}(s,a) &= Q_\theta^{(r)}(s,a) - \int Q_\theta^{(r)}(s,a')\pi_\theta(da'\mid s).
    \end{aligned}
\end{equation}
Combining this identity with~\eqref{e: fisher}, $w^*(\theta)$ is the least-squares coefficient obtained by regressing $A_\theta^{(r_\theta)}$ onto the score $u_\theta$.

\subsection{Pseudo-reward via Exact Dual Representation}
\label{s: dual.representation}

We next characterize the pseudo-reward through a dual representation that avoids explicit occupancy estimation. Let $\mathcal R$ denote the bounded measurable functions on $\mathcal S\times\mathcal A$.

\begin{prop}[Exact Dual Representation]\label{prop: dual}
    Under Assumption~\ref{a: utility}, there exists a concave functional $\Psi:\mathcal R\to\mathbb R\cup{-\infty}$ such that, for every policy under consideration,
    \begin{equation}\label{e: dual.objective}
        F(\lambda_\theta) = \min_{r\in\mathcal R} \left\{ \mathbb E_{\lambda_\theta}[r]-\Psi(r) \right\}, \qquad r_\theta \in \argmin_{r\in\mathcal R}\left\{
        \mathbb E_{\lambda_\theta}[r]-\Psi(r)
        \right\}.
    \end{equation}
\end{prop}
\noindent\emph{Proof.} For all $r \in \cR,$ define
\[
    \Psi(r)
    := \inf_{\mu \in \mathcal D}
       \left\{\int r\,d\mu-F(\mu)\right\}.
\]
The functional $\Psi$ is concave because it is an infimum
of affine functionals of $r$. For every $r \in \mathcal R$,
\[
    \Psi(r)
    \le \mathbb E_{\lambda_\theta}[r]-F(\lambda_\theta), 
\]
and hence
\[
    F(\lambda_\theta)
    \le \mathbb E_{\lambda_\theta}[r]-\Psi(r).
\]

By Assumption~\ref{a: utility}, $r_\theta=\nabla_\lambda F(\lambda_\theta)\in \cR.$ Concavity of $F$ gives, for every
$\mu \in \mathcal D$,
\[
    F(\mu)
    \le F(\lambda_\theta)
       + \int r_\theta\,d(\mu-\lambda_\theta).
\]
Rearranging, we have
\[
    \int r_\theta\,d\mu-F(\mu)
    \ge \mathbb E_{\lambda_\theta}[r_\theta]
       -F(\lambda_\theta),
\]
with equality at $\mu=\lambda_\theta$. Therefore, we get
\[
    \Psi(r_\theta)
    = \mathbb E_{\lambda_\theta}[r_\theta]
      -F(\lambda_\theta).
\]
Combining these inequalities proves
\[
    F(\lambda_\theta)
    = \min_{r \in \mathcal R}
      \left\{\mathbb E_{\lambda_\theta}[r]-\Psi(r)\right\}
    \quad \text{and} \quad
    r_\theta \in \arg\min_{r \in \mathcal R}
      \left\{\mathbb E_{\lambda_\theta}[r]-\Psi(r)\right\}.
\]
This proves Proposition~\ref{prop: dual}.
\proofend

Throughout the remainder of the paper, $\Psi$ denotes a
chosen concave dual satisfying~\eqref{e: dual.objective}, which need not coincide with the construction used in the proof of Proposition~1.
The stability and sampling conditions required of this
dual are stated separately in Section~\ref{s: assumptions}.

Proposition~\ref{prop: dual} outlines the variational RLGU representation of \citet{zhang2020variational}. Its key advantage is that the unknown occupancy enters only through an expectation. Since the dual objective touches $F$ at $r_\theta$, differentiability gives
\begin{equation}\label{e: dual.policy.gradient}
    \nabla_\theta J(\theta) = (1-\gamma)\left. \nabla_\vartheta V^{\pi_\vartheta}(r_\theta) \right|_{\vartheta=\theta}.
\end{equation}
Thus, learning an exact dual minimizer is sufficient to recover the pseudo-reward required by the actor.

To learn this pseudo-reward, our critic framework uses a
finite reference measure $\nu$ on $\mathcal S\times\mathcal A$,
with $\nu_0=\nu(\mathcal S\times\mathcal A)>0$, to define
its $L^2(\nu)$ geometry and dual-gradient oracle.
We first illustrate the dual construction for a separable
utility class. For this class, assume additionally that
$\lambda_\theta\ll\nu$ and write
$p_\theta=d\lambda_\theta/d\nu$.
\begin{definition}[Separable Concave Utility]\label{a: separable.utility}
    We call a utility separable with respect to measure $\nu$ if it has the form
    \begin{equation}\label{e: separable.utility}
    F(\lambda_\theta) = C - \int f(p_\theta)\ d\nu- \int c p_\theta\ d\nu,
\end{equation}
where $f$ is proper, closed, and convex, and $c$ is measurable. At attained density values in the interior of its effective domain, $f$ is differentiable; at an attained endpoint, $f'$ denotes a finite one-sided derivative in $\partial f$. All integrals are finite, and
$|f'(p_\theta)+c|$ is bounded uniformly in $\theta$, $\nu$-almost everywhere.
\end{definition}
\begin{remark}\label{r: separable.examples}
    Separable concave utilities cover visitation-entropy and relative-entropy objectives used for exploration \citep{hazan2019provably,zhang2020variational}, KL-based occupancy matching for imitation \citep{zhang2020variational}, more general $f$-divergence objectives \citep{nowozin2016fgan}, and quadratic density matching. The bounded-derivative condition holds under the corresponding density-range conditions summarized in Table~\ref{tab: utility.examples} in the appendix; details are given in Appendix~\ref{appendix: example.separable}.
\end{remark}

\subsection{Neural Critic Estimation}
\label{s: neural.critic}

We represent the pseudo-reward by an $L$-layer neural network. For bounded features $\phi:\mathcal S\times\mathcal A\to\mathbb R^p$, let $x^{(0)}=\phi(s,a)$ and
\begin{equation}\label{e: neural.critic}
    x^{(\ell)}=m^{-1/2}\sigma(W^{(\ell)}x^{(\ell-1)}),\quad \ell=1,\ldots,L,
    \qquad r(s,a,\zeta)=m^{-1/2}b^\top x^{(L)}.
\end{equation}
Here $W^{(1)}\in\mathbb R^{m\times p}$, $W^{(\ell)}\in\mathbb R^{m\times m}$ for $\ell\geq2$, and $b$ is fixed. The trainable weights form $\zeta\in\mathbb R^{D_m}$, where $D_m=mp+(L-1)m^2$.

To control the network locally around its initialization, we impose the following common assumption:
\begin{assump}[Smooth Activation]\label{a: sooth.activation}
    The activation $\sigma$ is Lipschitz with a Lipschitz derivative.
\end{assump}
Assumption~\ref{a: sooth.activation} is satisfied by commonly used smooth approximations of ReLU such as Sigmoid, ELU, and GeLU ~\cite{Hendrycks2016GaussianEL, Clevert2015FastAA}. Moreover, this ensures that the network admits a controlled first-order linearization around its initialization, as used in NTK analyses \citep{jacot2018neural,ganesh2025order,Satheesh2026GlobalCO}.

To control the critic using NTK analysis around its initialization, fix $\zeta_0$ and restrict the parameters to a ball $\Omega$ of radius $R$ centered at $\zeta_0$, and define
\begin{equation}\label{e: linearized.r.function}
    \begin{aligned}
        g_0(s,a) & := \nabla_\zeta r(s,a,\zeta_0),\qquad\mathcal F_{\Omega,m} :={r_{\rm lin}(\cdot,\zeta):\zeta\in\Omega},
        \\
        r_{\rm lin}(s,a,\zeta)& :=r(s,a,\zeta_0)+g_0(s,a)^\top(\zeta-\zeta_0).
    \end{aligned}
\end{equation}

We then require the initialization, tangent features, and linearization error to satisfy
\begin{multline}\label{e: local.ntk}
    \sup_{(s,a)}|r(s,a,\zeta_0)|\leq C_0,
    \qquad \sup_{(s,a)}\|g_0(s,a)\|_2\leq C_\phi, 
    \\
    \text{and} \qquad \sup_{(s,a),\zeta\in\Omega}|r(s,a,\zeta)-r_{\rm lin}(s,a,\zeta)|\leq\eta_m.
\end{multline}
The term $\eta_m$ captures the finite-width error. For example, a uniform parameter-Hessian bound $C_H/\sqrt m$ gives $\eta_m\leq C_HR^2/(2\sqrt m)$ by Lemma~\ref{app.lem: NTK.bounds}.

\section{Algorithm}
\label{s: algorithm}

\begin{wrapfigure}{r}{0.49\textwidth}
\begin{minipage}{0.47\textwidth}
\hrule
\vspace{0.35em}
\textbf{Algorithm 1: Neural Dual-Functional NAC}
\vspace{0.25em}

\footnotesize
\begin{algorithmic}[1]
\Require $\theta_0,\zeta_0,K,B,H_{\max}=2^J,
\alpha,\{\beta_{\zeta,h}\},\beta_w,H_w$
\For{$k=0,\ldots,K-1$}
    \State Draw $N_k\sim{\rm Geom}(1/2)$ and set
    $H_k=2^{N_k}$ if $N_k\leq J$, otherwise $H_k=1$
    \State Initialize $\zeta_0^k=\zeta_0$
    \For{$h=0,\ldots,H_k-1$}
        \State Sample a fresh length-$B$ trajectory under $\pi_{\theta_k}$
        and an independent dual-oracle draw
        \State Update $\zeta_{h+1}^k$ by~\eqref{e: critic.update}
    \EndFor
    \State If $N_k\leq J$, retain the shared-run outputs
    $R_{k,1}$, $R_{k,2^{N_k-1}}$, $R_{k,2^{N_k}}$;
    otherwise retain $R_{k,1}$
    \State Form $r_{k,\mathrm{MLMC}}$
    by~\eqref{e: function .critic.MLMC}
    \State Initialize $w_0^k=0$
    \For{$h=0,\ldots,H_w-1$}
        \State Sample a fresh length-$B$ trajectory under $\pi_{\theta_k}$
        \State Compute $\widehat{\mathcal I}^{k,h}$ and
        $\widehat g^{k,h}$ by~\eqref{e: NPG.estimator}
        \State Update $w_{h+1}^k$ by~\eqref{e: NPG.update}
    \EndFor
    \State Set $w_k=w_{H_w}^k$ and
    $\theta_{k+1}=\theta_k+\alpha w_k$
\EndFor
\end{algorithmic}

\vspace{0.35em}
\hrule
\end{minipage}
\end{wrapfigure}

At outer iteration $k$, we hold $\pi_{\theta_k}$ fixed, learn a dual reward, estimate a natural-gradient direction, and update
\begin{equation}\label{e: actor.update}
\theta_{k+1}=\theta_k+\alpha w_k,
\qquad
k=0,\ldots,K-1.
\end{equation}
Policy trajectories are independently restarted from $\rho$. States within each trajectory follow the MDP and may be dependent. Conditional on the current policy, critic trajectories, NPG trajectories, and reference samples use independent randomness. One environment transition counts as one unit of trajectory cost.

\subsection{Critic Estimator}
\label{s: critic.subroutine}

For fixed $\theta_k$, the critic minimizes the dual objective
\[
    \mathcal E(\theta_k,\zeta)
    =
    \mathbb E_{\lambda_{\theta_k}}[r(\cdot,\zeta)]
    -
    \Psi(r(\cdot,\zeta)).
\]

We use the frozen tangent feature $g_0$ from~\eqref{e: linearized.r.function} in the critic update, while evaluating the nonlinear reward inside the dual derivative. The resulting error is controlled by the local linearization bound~\eqref{e: local.ntk}.

\paragraph{MLMC Critic Function.} Choose $H_{\max}=2^J$ and draw $N_k\in{1,2,\ldots}$ with $\mathbb P(N_k=j)=2^{-j}$. If $N_k\leq J$, run the critic for $2^{N_k}$ updates and retain the outputs at horizons $1$, $2^{N_k-1}$, and $2^{N_k}$; otherwise, run one update. Writing $R_{k,H}=r(\cdot,\zeta_H^k)$, define
\begin{equation}\label{e: function .critic.MLMC}
r_{k,\mathrm{MLMC}}
=
R_{k,1}
+
\mathbf 1_{{N_k\leq J}}2^{N_k}
\left(
R_{k,2^{N_k}} - R_{k,2^{N_k-1}} \right).
\end{equation}
All retained outputs are generated from the same critic run, so the dyadic levels share the required prefixes. The correction acts on critic functions and may lie outside the neural class; its bias and second moment are controlled in the analysis.

\paragraph{Vanilla Critic Subroutine.}
To estimate the dual-gradient term, assume access at each $\zeta\in\Omega$ to a random vector $Z_\Psi(\zeta)$ satisfying
\begin{equation}\label{e: dual.oracle}
    \mathbb E[Z_\Psi(\zeta)\mid\zeta]
    =
    \int
    g_0(s,a)\,
    \nabla_r\Psi(r(\cdot,\zeta))(s,a)
    \,d\nu(s,a)
    \quad \text{and} \quad
    \mathbb E[\|Z_\Psi(\zeta)\|_2^2\mid\zeta]
    \leq
    S_\Psi^2.
\end{equation}
For the utilities in Table~\ref{tab: utility.examples}, one or two independent reference draws provide such an estimator; explicit formulas are given in Proposition~\ref{prop: separable.dual}. Exact integration is a special case, and reference draws are counted separately from environment transitions.

Initialize $\zeta_0^k=\zeta_0$. For $h=0,1,\ldots$, generate a fresh length-$B$ trajectory under $\pi_{\theta_k}$ and an independent dual-oracle draw, and update
\begin{equation}\label{e: critic.update}
\zeta_{h+1}^k
=
\Pi_\Omega
\left(
\zeta_h^k-\beta_{\zeta,h}\widehat v^{k,h}
\right),
\quad \text{where} \quad
\widehat v^{k,h}
=
(1-\gamma)\sum_{t=0}^{B-1}\gamma^t
g_0(s_t^{k,h},a_t^{k,h})
-
Z_\Psi(\zeta_h^k).
\end{equation}
Here $\beta_{\zeta,h}=\frac{2}{\mu(h+2)}$ with $\mu=\mu_c\kappa$, where $\mu_c$ and $\kappa$ are defined in Section~\ref{s: assumptions}. The decreasing stepsize gives uniform control across the dyadic horizons used by the MLMC estimator, while projection keeps the critic within the region covered by~\eqref{e: local.ntk}.

\subsection{NPG Estimator}
\label{s: npg.estimator}

Hold the sampled function $r_{k,\mathrm{MLMC}}$ fixed throughout. Start at $w_0^k=0$. At each of $H_w$ iterations, generate a fresh length-$B$ trajectory, independent of critic training, and compute
\begin{align}\label{e: NPG.estimator}
\widehat{\mathcal I}^{k,h} & =(1-\gamma)\sum_{t=0}^{B-1}\gamma^t u_{\theta_k}(s_t^{k,h},a_t^{k,h})u_{\theta_k}(s_t^{k,h},a_t^{k,h})^\top,
\\
\widehat g^{k,h} & =(1-\gamma)\sum_{t=0}^{B-1}\gamma^t r_{k,\mathrm{MLMC}}(s_t^{k,h},a_t^{k,h})\sum_{j=0}^t u_{\theta_k}(s_j^{k,h},a_j^{k,h}).\nonumber
\end{align}
Apply the stochastic quadratic update
\begin{equation}\label{e: NPG.update}
w_{h+1}^k=w_h^k-\beta_w\big(\widehat{\mathcal I}^{k,h}w_h^k-\widehat g^{k,h}\big),
\qquad w_k=w_{H_w}^k.
\end{equation}
The Fisher and gradient estimates may use the same trajectory. Conditional independence across iterations yields a linear mean recursion, so the actor bias depends on the \emph{mean} critic error while the second-moment analysis captures the full MLMC randomness.

\section{Assumptions and Main Result}
\label{s: assumptions}

Before we state our main result in Theorem~\ref{thm: main.result}, we present our key assumptions.

\begin{assump}[KL Potential]\label{a: initial.kl}
The comparison policy $\pi_{\theta^*}$ satisfies
\[        
    \Phi_0:=
        \mathbb E_{s\sim d_{\theta^*}}
        \mathrm{KL}\big(
        \pi_{\theta^*}(\cdot\mid s)
        \|
        \pi_{\theta_0}(\cdot\mid s)
        \big)
        <\infty.
\]
\end{assump}
%

\begin{assump}[Score Function]\label{a: score.function}
For all relevant parameters and state-action pairs,
\[        
    \|u_\theta(s,a)\|_2\leq G_1,\qquad
    \|u_\theta(s,a)-u_{\theta'}(s,a)\|_2\leq G_2\|\theta-\theta'\|_2.
\]
\end{assump}
%

\begin{assump}[Fisher Non-Degeneracy]\label{a: Fisher.degenrate}
    There exists $\mu_F>0$ such that
    $\mathcal I(\theta)\succeq\mu_F I_d$ for all policies under consideration.
\end{assump}
%

\begin{assump}[Transferred Compatibility]\label{a: FA.error.PG}
    Let $w^*(\theta)$ be the on-policy natural-gradient coefficient in~\eqref{e: fisher}. For some $\varepsilon_{\rm bias}\geq0$,
    \begin{equation}\label{e: transfer.assumption}
        \mathcal L_{\lambda_{\theta^*}}(\theta,w^*(\theta)):=
        \mathbb E_{\lambda_{\theta^*}}
        \left[
        \big( u_\theta(s,a)^\top w^*(\theta)- A_\theta^{(r_\theta)}(s,a)\big)^2\right]\leq\varepsilon_{\rm bias}.
    \end{equation}
\end{assump}
%

We next state the critic conditions. All functional norms below are in $L_2(\nu)$.

\begin{assump}[Uniform Dual Stability]\label{a: dual.stability}
    The functional $\Psi$ is differentiable, $\mu_c$-strongly concave, and $L_c$-smooth on a convex set containing
    $\{r(\cdot,\zeta):\zeta\in\Omega\}$ and $\mathcal F_{\Omega,m}$. The oracle condition~\eqref{e: dual.oracle} holds uniformly on $\Omega$.
\end{assump}
%

\begin{assump}[Gram Matrix Non-Degeneracy]\label{a: critic.gram}
    The tangent-feature Gram matrix
    \begin{equation}\label{e: critic.gram}
        G=\int g_0(s,a)g_0(s,a)^\top\,\nu(ds,da)
    \end{equation}
    satisfies $G\succeq\kappa I_{D_m}$ for some $\kappa>0$.
\end{assump}

Define the linearized-dual objective and its minimizer by
\begin{equation}\label{e: linear.dual.target}
    \mathcal E^{\rm lin}(\theta,\zeta)
    :=
    \mathbb E_{\lambda_\theta}[r_{\rm lin}(\cdot,\zeta)]
    -
    \Psi(r_{\rm lin}(\cdot,\zeta)),
    \quad
    \zeta_\theta^*
    :=
    \argmin_{\zeta\in\Omega}\mathcal E^{\rm lin}(\theta,\zeta),
    \quad
    \bar r_\theta
    :=
    r_{\rm lin}(\cdot,\zeta_\theta^*).
\end{equation}

\begin{assump}[Critic Approximation Error]\label{a: neural.critic.error}
    For some $\varepsilon_{\rm app}\geq0$, uniformly over the policies considered,
    \begin{equation}\label{e: app.assumption}
        \|r_\theta-\bar r_\theta\|_{L_2(\lambda_\theta)}^2
        \leq
        \varepsilon_{\rm app}.
    \end{equation}
\end{assump}

\paragraph{Discussion of assumptions.}
Assumptions~\ref{a: initial.kl}--\ref{a: FA.error.PG} are standard in global analyses of approximate natural policy gradient \citep{agarwal2021theory,ganesh2025order,ganesh2025sharper}. The KL potential supports comparison with $\pi_{\theta^*}$; bounded and Lipschitz scores control likelihood-ratio moments and the KL remainder; and Fisher non-degeneracy ensures a well-conditioned natural-gradient direction. Transferred compatibility captures approximation error and distribution shift to the comparator occupancy; an on-policy residual bound $e$ together with $\lambda_{\theta^*}\leq C\lambda_\theta$ gives $\varepsilon_{\rm bias}\leq Ce$ \citep{agarwal2021theory,ganesh2025order}. Assumptions~\ref{a: dual.stability}--\ref{a: neural.critic.error} govern the dual critic. Strong concavity and smoothness stabilize the variational reward-learning problem
\citep{zhang2020variational,nowozin2016fgan}, while Gram non-degeneracy provides
the required feature conditioning \citep{ganesh2025order,Satheesh2026GlobalCO}.
Together, they give the critic strong-convexity constant $\mu=\mu_c\kappa$.
Moreover,
\begin{equation}\label{e: gram.trace}
    D_m\kappa\leq\operatorname{tr}G\leq\nu_0C_\phi^2,
\end{equation}
so the conditioning may deteriorate as the critic dimension grows. For finite
spaces, $\operatorname{rank}G\leq|\mathcal S||\mathcal A|$, and hence
Assumption~\ref{a: critic.gram} cannot hold with $\kappa>0$ once
$D_m>|\mathcal S||\mathcal A|$. Finally, $\varepsilon_{\rm app}$ isolates pseudo-reward approximation error, paralleling occupancy-model realizability assumptions \citep{barakat2025global}; Proposition~\ref{prop: approximation.transfer} relates it to $L_2(\nu)$ approximation. These conditions separate policy approximation, critic approximation, conditioning, and finite-width effects in the final bound.

Under Assumptions~\ref{a: utility}--\ref{a: neural.critic.error}, we have the following main result:
%

\begin{theorem}[Global convergence]\label{thm: main.result}
    Fix the critic representation $(m,R,\zeta_0)$ and the constants in~\eqref{e: local.ntk}. Suppose Assumptions~\ref{a: utility}, \ref{a: sooth.activation}, and~\ref{a: initial.kl}--\ref{a: neural.critic.error} hold under the sampling model of Section~\ref{s: algorithm}. For $T\geq2$, choose $K=T$, $H_{\max}=2^{\lceil\log_2T\rceil}$, and
    \begin{align}\label{e: main.parameters}
        B &=\left\lceil\frac{\max\{2\log T,\log(2G_1^2/\mu_F)\}}{\log(1/\gamma)}\right\rceil,\qquad 0<\beta_w
        \leq \min \left\{\frac{1}{G_1^2},\frac{\mu_F}{8G_1^4}\right\},
        \nonumber\\
        H_w &= \left\lceil\frac{4\log T}{\mu_F\beta_w}\right\rceil,
        \qquad
        V_T = 1+\log(2H_{\max})+ H_{\max}(\eta_m^2+\gamma^{2B}),
        \nonumber\\
        \beta_{\zeta,h} &= \frac{2}{\mu(h+2)}, \qquad \alpha =  \frac{1-\gamma}{\sqrt{KV_T}}.
    \end{align}
    Then there exists a constant $C$, independent of $T$, such that
    \begin{equation}\label{e: main.result}
        \gapK :=F^*- \frac1K \sum_{k=0}^{K-1} \mathbb E F(\lambda_{\theta_k}) \leq \sqrt{\varepsilon_{\rm bias}} + \frac{C}{1-\gamma}
        \left( \sqrt{\varepsilon_{\rm app}}+ \eta_m + \sqrt{\frac{1+\log T}{T}} \right).
    \end{equation}
\end{theorem}

\begin{remark}[Expected Sample Complexity]\label{rem: rate}
The expected environment-transition cost is $\mathcal O(KB[H_w+1+\log H_{\max}])$. For fixed representation and problem constants, Theorem~\ref{thm: main.result} therefore gives $\widetilde{\mathcal O}(\epsilon^{-2})$ expected sample complexity above the approximation floor
$\sqrt{\varepsilon_{\rm bias}}+C(\sqrt{\varepsilon_{\rm app}}+\eta_m)/(1-\gamma)$. An $\mathcal O(m^{-1/2})$ Hessian bound gives $\eta_m=\mathcal O(m^{-1/2})$ for fixed $R$. If $m$ grows with $T$, the width dependence of $C$, particularly through $\kappa$ in~\eqref{e: gram.trace}, must also be tracked. Reference samples are counted separately, while representation-dependent neural and matrix computations are excluded.
\end{remark}

\section{Proof Outline}
\label{s: proof.outline}

In this section, we outline the proof of Theorem~\ref{thm: main.result}, which proceeds in four key steps. Before presenting them, let $\mathcal F_k$ denote the $\sigma$-field generated by all randomness prior to outer iteration $k$, and write $\mathbb E_k[\cdot]=\mathbb E[\cdot\mid\mathcal F_k]$. We also fix a critic initialization satisfying~\eqref{e: local.ntk}.

\paragraph{Step 1.~General Framework for NPG:}
Combining the KL-potential argument
\citep{agarwal2021theory,ganesh2025order,ganesh2025sharper} with concavity and transferred compatibility yields the following bound, proved in Appendix~\ref{appendix: opt.gap}.
\begin{lemma}[Optimality Gap]\label{lem: opt.gap}
    Under Assumptions~\ref{a: utility} and~\ref{a: initial.kl}--\ref{a: FA.error.PG},
    \begin{equation}\label{e: actor.gap}
        \gapK\leq\sqrt{\varepsilon_{\rm bias}}+\frac{\Phi_0}{\alpha K}+\frac{G_1}{K}\sum_{k=0}^{K-1} \mathbb E\|\mathbb E_k w_k-w^*(\theta_k)\|_2 +\frac{G_2\alpha}{2K}\sum_{k=0}^{K-1}\mathbb E\|w_k\|_2^2.
    \end{equation}
\end{lemma}
Actor bias enters linearly, while the second moment is weighted by $\alpha,$ allowing a noisy MLMC estimator with an accurate mean, at the cost of a logarithmically increasing second moment.

\paragraph{Step 2.~NPG Analysis:}
For the trained pseudo-reward, define
\begin{equation}\label{e: critic.errors}
    \delta_k^2=\|\mathbb E_k r_{k,\mathrm{MLMC}}-\bar r_{\theta_k}\|_{L_2(\lambda_{\theta_k})}^2,
    \qquad
    \sigma_k^2=\mathbb E_k\|r_{k,\mathrm{MLMC}}-\bar r_{\theta_k}\|_{L_2(\lambda_{\theta_k})}^2.
\end{equation}
The conditional expectation in $\sigma_k^2$ includes critic training and the random level. The NPG trajectories are sampled afterward. Given the frozen critic, their mean update is a linear recursion with matrix $\mathcal I_B(\theta_k)=\mathbb E\widehat{\mathcal I}^{k,h}$. Its difference from $\mathcal I(\theta_k)$ is at most $G_1^2\gamma^B$. Crucially, the gradient estimate is linear in the supplied reward. This transfers $\delta_k$ to the actor bias rather than replacing it by $\sigma_k$, giving us the following result:

\begin{lemma}[NPG Bounds]\label{thm: NPG.bound}
    Suppose the bounded reward, score, Fisher, and critic-approximation conditions hold, $G_1^2\gamma^B\leq\mu_F/2$, and $\beta_w$ satisfies~\eqref{e: main.parameters}. Let
    $\tau_B=\gamma^B(B+(1-\gamma)^{-1})$. For a constant $C$ depending only on the fixed regularity bounds,
    \begin{align}\label{e: npg.bounds}
        \|\mathbb E_k w_k-w^*(\theta_k)\|_2^2
        &\leq C\left[\tau_B^2+
        \frac{\varepsilon_{\rm app}+\delta_k^2+e^{-\mu_F\beta_w H_w}}{(1-\gamma)^2}\right],
        \nonumber\\
        \mathbb E_k\|w_k\|_2^2+\mathbb E_k\|w_k-w^*(\theta_k)\|_2^2
        &\leq\frac{C}{(1-\gamma)^2}(1+\sigma_k^2).
    \end{align}
\end{lemma}
Appendix~\ref{appendix: NPG.analysis} proves these bounds using a reset-trajectory specialization of the general linear stochastic-recursion result in Lemma~\ref{app.lem: gen.linear}. The required second moment is bounded rather than vanishing with $H_w$; logarithmically many iterations are used to reduce the mean initialization error.

\paragraph{Step 3.~Dual-Critic Analysis:}
We next control $\delta_k$ and $\sigma_k$. Write
$\zeta_k^*=\zeta_{\theta_k}^*$ and
$q_m=\eta_m^2+\gamma^{2B}$. Trajectory truncation and neural linearization induce the conditional gradient bias
\[
    d_v=C_\phi\gamma^B+\nu_0C_\phi L_c\eta_m,
\]
while the reference-oracle noise remains conditionally unbiased with bounded second moment. Dual curvature, the Gram condition $G\succeq\kappa I$, and the constrained minimizer's variational inequality yield a projected-SGD contraction with effective curvature $\mu_c\kappa$. Transferring the resulting parameter error through the tangent features, we prove the following result in Appendix~\ref{appendix: dual.critic.analysis}:
\begin{lemma}[Vanilla-Critic Bounds]\label{thm: vanilla.critic.bound}
    Under~\eqref{e: local.ntk}, \eqref{e: dual.oracle}, and Assumptions~\ref{a: dual.stability}--\ref{a: critic.gram}, the updates~\eqref{e: critic.update} satisfy, for all $H\geq1$,
    \begin{equation}\label{e: vanilla.bound}
        \max\left\{
        \mathbb E_k\|\zeta_H^k-\zeta_k^*\|_2^2,\,
        \mathbb E_k\|R_{k,H}-\bar r_{\theta_k}\|_{L_2(\lambda_{\theta_k})}^2\right\}
        \leq\frac{C}{H+1}+Cq_m.
    \end{equation}
\end{lemma}
Thus, a vanilla critic trained for $H$ steps has squared error of order
$\mathcal O(H^{-1})$ above the approximation floor. Consequently, obtaining the
$\widetilde{\mathcal O}(T^{-1})$ squared critic bias required by the actor would
necessitate $H=\varTheta(T)$ critic updates at every outer iteration. Since the
actor itself performs $K=\varTheta(T)$ outer iterations, using such a critic
directly would lead to a quadratic-order critic-training cost and would therefore
preclude the desired $\widetilde{\mathcal O}(T)$ overall sampling complexity.
The purpose of the MLMC construction is to retain the mean accuracy of a critic
trained for $H_{\max}$ steps while reducing its expected optimization cost to
logarithmic order.

Rather than correcting a scalar statistic, we exploit the MLMC construction in~\eqref{e: function .critic.MLMC}. Recall that $H_{\max}=2^J$ and
$N_k\sim\textnormal{Geom}(1/2)$. This gives an expected training
horizon of
\[
    \mathbb E[H_k]
    =
    \sum_{j=1}^{J}2^{-j}2^j
    +
    \mathbb P(N_k>J)
    =
    J+2^{-J}
    \leq
    1+\log_2 H_{\max}.
\]
At the same time, the dyadic correction in~\eqref{e: function .critic.MLMC}
telescopes in function space, so its conditional expectation matches that of
the critic trained for the maximum horizon $H_{\max}$. Consequently, the
long-horizon bound in Lemma~\ref{thm: vanilla.critic.bound} transfers directly
to the conditional bias quantity $\delta_k$. In contrast, the second moment
$\sigma_k^2$ depends on the random MLMC correction itself; controlling it
requires bounds on the differences between critic outputs at adjacent dyadic
horizons. Lemma~\ref{thm: vanilla.critic.bound} provides precisely these
level-wise estimates. Combining the telescoping identity, the dyadic-difference
bounds, and the logarithmic expected horizon established above yields the
following MLMC result in Appendix~\ref{appendix: dual.critic.analysis}:

\begin{lemma}[MLMC-Critic Bounds]\label{thm: MLMC.critic.bound}
    Under Lemma~\ref{thm: vanilla.critic.bound}, the estimator \eqref{e: function .critic.MLMC} satisfies
    \begin{equation}\label{e: mlmc.bounds}
        \delta_k^2\leq C(H_{\max}^{-1}+q_m),\qquad
        \sigma_k^2\leq C[1+\log H_{\max}+H_{\max}q_m],
    \end{equation}
    using at most $1+\log_2H_{\max}$ expected critic updates.
\end{lemma}
Hence the MLMC construction exchanges the $\mathcal O(H_{\max})$ cost of a long vanilla critic for $\mathcal O(\log H_{\max})$ expected cost, while preserving its conditional mean and incurring only a controlled increase in the second moment.

\paragraph{Step 4.~Matching the Bias and Variance Scales:}
Finally, choose $H_{\max}\asymp K=T$ and
$\alpha=(1-\gamma)/\sqrt{KV_T}$. Lemma~\ref{thm: MLMC.critic.bound} gives the long-horizon critic bias at logarithmic expected inner cost, while Lemma~\ref{thm: NPG.bound} converts this conditional critic bias into actor bias and confines the larger MLMC second moment to the $\alpha$-weighted term in~\eqref{e: actor.gap}. Substitution yields the claimed rate because
\[
    \sqrt{V_T/T}
    =
    \mathcal O\left(
    \sqrt{\frac{1+\log T}{T}}
    +\eta_m+\gamma^B
    \right).
\]
Thus the statistical term retains the
$\widetilde{\mathcal O}(T^{-1/2})$ rate, while critic approximation and finite-width errors appear only through the final approximation floor. Details are in Appendix~\ref{appendix: main.result}.





\bibliography{iclr2027_conference}
\bibliographystyle{iclr2027_conference}
\clearpage
\appendix

\section{Optimality Gap: Proof of Lemma~\ref{lem: opt.gap}}
\label{appendix: opt.gap}

The proof adapts the KL-potential argument for approximate NPG \citep{agarwal2021theory,ganesh2025order}. We give it in full because the conditional mean, rather than the mean norm of the random critic error, is needed here. Define
\[
 \Phi_k=\mathbb E_{s\sim d_{\theta^*}}\mathrm{KL}(\pi_{\theta^*}(\cdot\mid s)\|\pi_{\theta_k}(\cdot\mid s)),
 \qquad b_k=\mathbb E_{\lambda_{\theta^*}}u_{\theta_k}(s,a).
\]
Assumption~\ref{a: score.function} and the actor update imply, pathwise,
5
\begin{align}\label{e: appendix.kl}
    \Phi_k-\Phi_{k+1}
    &=\mathbb E_{\lambda_{\theta^*}}\big[\log\pi_{\theta_{k+1}}(a\mid s)-\log\pi_{\theta_k}(a\mid s)\big]
    \nonumber\\
    &\geq\alpha b_k^\top w_k-\frac{G_2\alpha^2}{2}\|w_k\|_2^2.
\end{align}
The finite initial KL and bounded score ensure the displayed difference is integrable whenever the update has a finite first moment; the NPG second-moment bound supplies this property for the algorithm.

By Lemma~\ref{app.lem: performance.diff} and transferred compatibility,
\begin{align}\label{e: appendix.comparator}
    F^*-F(\lambda_{\theta_k})
    &\leq\mathbb E_{\lambda_{\theta^*}}A_{\theta_k}^{(r_{\theta_k})}
    \nonumber\\
    &\leq b_k^\top w^*(\theta_k)
    +\left(\mathbb E_{\lambda_{\theta^*}}
    [A_{\theta_k}^{(r_{\theta_k})}-u_{\theta_k}^\top w^*(\theta_k)]^2\right)^{1/2}
    \nonumber\\
    &\leq b_k^\top w^*(\theta_k)+\sqrt{\varepsilon_{\rm bias}}.
\end{align}
Both $b_k$ and $w^*(\theta_k)$ are $\mathcal F_k$-measurable, and $\|b_k\|_2\leq G_1$. Taking conditional expectations in~\eqref{e: appendix.kl}, inserting~\eqref{e: appendix.comparator}, and rearranging gives
\begin{align*}
    F^*-F(\lambda_{\theta_k})
    \leq{}&\sqrt{\varepsilon_{\rm bias}}
    +\frac{\Phi_k-\mathbb E_k\Phi_{k+1}}{\alpha}
    +G_1\|\mathbb E_k w_k-w^*(\theta_k)\|_2
    +\frac{G_2\alpha}{2}\mathbb E_k\|w_k\|_2^2.
\end{align*}
Take expectations and sum from $k=0$ to $K-1$. The potentials telescope, and $\Phi_K\geq0$. Division by $K$ yields~\eqref{e: actor.gap}. This proves Lemma~\ref{lem: opt.gap}. \proofend

\section{NPG Analysis: Proof of Lemma~\ref{thm: NPG.bound}}
\label{appendix: NPG.analysis}

This is a reset-trajectory adaptation of the conditional-bias approach in \citet{ganesh2025sharper}. We prove the needed reward-to-gradient bounds and the stochastic solve explicitly. Fix $k$ and condition on $\mathcal F_k$ throughout. Write $\theta=\theta_k$, $\lambda=\lambda_\theta$, $r^\circ=r_\theta$, $\bar r=\bar r_\theta$, and $\widehat r=r_{k,\mathrm{MLMC}}$. The random function $\widehat r$ is sampled before the NPG trajectories.

Let $A=\mathcal I_B(\theta)=\mathbb E\widehat{\mathcal I}^{k,h}$ and $I_\theta=\mathcal I(\theta)$. The missing tail is positive semidefinite, and
\begin{equation}\label{e: Fisher.bias}
    0\preceq I_\theta-A\preceq G_1^2\gamma^B I_d,
    \qquad \|I_\theta-A\|_{\rm op}\leq G_1^2\gamma^B.
\end{equation}
Consequently, $A\succeq aI_d$ with $a=\mu_F/2$ under the theorem's truncation condition. Every sampled matrix obeys
\begin{equation}\label{e: Fisher.MSE}
 0\preceq\widehat{\mathcal I}^{k,h}\preceq G_1^2I_d.
\end{equation}
Both statements follow directly from the bounded score and the normalized discounted weights.

For a deterministic square-integrable reward $v$, let
\begin{equation}\label{e: gradient.map}
    \mathcal Q_B(v)=(1-\gamma)\mathbb E_{\pi_\theta}
    \sum_{t=0}^{B-1}\gamma^t v(s_t,a_t)\sum_{j=0}^t u_\theta(s_j,a_j).
\end{equation}
Put $C_\gamma=G_1\sqrt{1+\gamma}/(1-\gamma)$. The identity
\[
 (1-\gamma)\sum_{t\geq0}\gamma^t(t+1)^2=\frac{1+\gamma}{(1-\gamma)^2}
\]
and Cauchy-Schwarz imply
\begin{equation}\label{e: gradient.map.bound}
    \|\mathcal Q_B(v)\|_2\leq C_\gamma\|v\|_{L_2(\lambda)}.
\end{equation}
Indeed, the norm of the summand's score sum is at most $G_1(t+1)$; apply Cauchy-Schwarz to the finite measure assigning mass $(1-\gamma)\gamma^t$ to time $t$ and then enlarge the nonnegative second sum to all times. The same argument pathwise, followed by expectation, proves the second-moment inequality
\begin{equation}\label{e: gradient.sample.second}
    \mathbb E\left\|(1-\gamma)\sum_{t=0}^{B-1}\gamma^t v(s_t,a_t)\sum_{j=0}^t u_\theta(s_j,a_j)\right\|_2^2
    \leq C_\gamma^2\|v\|_{L_2(\lambda)}^2.
\end{equation}
It assumes no independence between times within a single trajectory.

The true reward is bounded, so the gradient tail satisfies
\begin{equation}\label{e: PG.tail}
    \|\mathcal Q_B(r^\circ)-\nabla_\theta J(\theta)\|_2
    \leq G_1R_{\max}\gamma^B\left(B+\frac{1}{1-\gamma}\right).
\end{equation}
This follows from
$(1-\gamma)\sum_{t=B}^\infty\gamma^t(t+1)=\gamma^B[B+(1-\gamma)^{-1}].$

Fresh NPG trajectories are independent of $\widehat r$. Therefore, the tower property and linearity give
$\mathbb E_k\widehat g^{k,h}=\mathcal Q_B(\mathbb E_k\widehat r)$ for each inner iteration after averaging over its past. Decompose
\[
    \mathbb E_k\widehat r-r^\circ=(\mathbb E_k\widehat r-\bar r)+(\bar r-r^\circ).
\]
By~\eqref{e: gradient.map.bound},~\eqref{e: PG.tail}, and Assumption~\ref{a: neural.critic.error},
\begin{equation}\label{e: PG.bias.strong}
    \|\mathcal Q_B(\mathbb E_k\widehat r)-\nabla_\theta J(\theta)\|_2
    \leq G_1R_{\max}\tau_B+C_\gamma(\delta_k+\sqrt{\varepsilon_{\rm app}}).
\end{equation}
For the second moment,~\eqref{e: local.ntk} gives $\|\bar r\|_\infty\leq C_r:=C_0+C_\phi R$. Then
\begin{equation}\label{e: PG.MSE}
    \mathbb E_k\|r_{k,\mathrm{MLMC}}\|_{L_2(\lambda)}^2
    \leq2\sigma_k^2+2C_r^2,
    \quad
    \mathbb E_k\|\widehat g^{k,h}\|_2^2
    \leq2C_\gamma^2(\sigma_k^2+C_r^2).
\end{equation}
The second inequality also holds conditional on a fixed critic with $\|\widehat r\|_{L_2(\lambda)}^2$ in place of the averaged bound.

Condition further on the sampled critic. The hypotheses of Lemma~\ref{app.lem: gen.linear} hold with $L=G_1^2$, $a=\mu_F/2$, $q=\mathcal Q_B(\widehat r)$, and $V=C_\gamma^2\|\widehat r\|_{L_2(\lambda)}^2$. Hence
\[
    \mathbb E[w_{H_w}^k\mid\mathcal F_k,\widehat r]
    =\big[I_d-(I_d-\beta_wA)^{H_w}\big]A^{-1}\mathcal Q_B(\widehat r).
\]
The matrix in this expression is independent of $\widehat r$. Averaging over the critic thus replaces $\widehat r$ by $\mathbb E_k\widehat r$ exactly. Let $M=I_d-\beta_wA$, $w^*=I_\theta^{-1}\nabla_\theta J(\theta)$, and $\bar q=\mathcal Q_B(\mathbb E_k\widehat r)$. Then
\begin{equation}\label{e: NPG.mean.identity}
    \mathbb E_kw_k-w^*
    =-M^{H_w}w^*+(I_d-M^{H_w})A^{-1}(\bar q-Aw^*).
\end{equation}
Since $\|M^{H_w}\|_{\rm op}\leq e^{-a\beta_wH_w}$, $\|(I_d-M^{H_w})A^{-1}\|_{\rm op}\leq a^{-1}$, and
\[
    \|w^*\|_2\leq W_*:=\frac{G_1R_{\max}}{\mu_F(1-\gamma)},
\]
we obtain
\begin{align}\label{e: NPG.bias.explicit}
    \|\mathbb E_kw_k-w^*\|_2\leq{}&e^{-a\beta_wH_w}W_*
    +a^{-1}\left[G_1R_{\max}\tau_B+C_\gamma(\delta_k+\sqrt{\varepsilon_{\rm app}})+G_1^2\gamma^BW_*\right].
\end{align}
Squaring and using $2a=\mu_F$ proves the first line of~\eqref{e: npg.bounds}; the extra $\gamma^B/(1-\gamma)$ term is bounded by a constant times $\tau_B$.

Finally, the second-moment part of Lemma~\ref{app.lem: gen.linear}, averaged over the critic, yields
\[
    \mathbb E_k\|w_k\|_2^2
    \leq\left(\frac{2}{a^2}+\frac{4\beta_w}{a}\right)
    C_\gamma^2(2\sigma_k^2+2C_r^2).
\]
The inequality $\|w_k-w^*\|^2\leq2\|w_k\|^2+2W_*^2$ establishes the second line of~\eqref{e: npg.bounds}. All constants depend only on the fixed representation and policy bounds. This proves Lemma~\ref{thm: NPG.bound}. \proofend

\section{Dual Critic Analysis: Proofs of Lemmas~\ref{thm: vanilla.critic.bound} and~\ref{thm: MLMC.critic.bound}}
\label{appendix: dual.critic.analysis}

\paragraph{Proof of Lemma~\ref{thm: vanilla.critic.bound}:}
Fix $k$ and condition on $\mathcal F_k$. Write $f_k(\zeta)=\mathcal E^{\rm lin}(\theta_k,\zeta)$ and $\zeta^*=\zeta_k^*$. The linear map $J_0$ satisfies
\begin{equation}\label{e: J0.bound}
    \|J_0\|_{\mathbb R^{D_m}\to L_2(\nu)}
    =\|J_0^*\|_{L_2(\nu)\to\mathbb R^{D_m}}
    \leq\sqrt{\nu_0}\,C_\phi.
\end{equation}
The composition of the strongly convex functional $-\Psi$ with this affine map is $\mu=\mu_c\kappa$ strongly convex because $\|J_0v\|_{L_2(\nu)}^2=v^\top Gv\geq\kappa\|v\|_2^2$. The linear occupancy term preserves that curvature. Consequently, $f_k$ has a unique minimizer on the compact convex set $\Omega$.

Let $\mathcal F_{k,h}$ contain the critic history before update $h$. Direct differentiation of the linearized objective and the oracle mean gives
\begin{align}\label{e: critic.gradient.bias.identity}
    e_h&:=\mathbb E[\widehat v^{k,h}\mid\mathcal F_{k,h}]-\nabla f_k(\zeta_h^k)
    \nonumber\\
    &=-(1-\gamma)\sum_{t=B}^\infty\gamma^t\mathbb E_{\pi_{\theta_k}}g_0(s_t,a_t)
    +J_0^*\big[\nabla\Psi(r_{\rm lin}(\cdot,\zeta_h^k))-\nabla\Psi(r(\cdot,\zeta_h^k))\big].
\end{align}
Using~\eqref{e: J0.bound}, dual smoothness on the set containing both arguments, and~\eqref{e: local.ntk},
\begin{equation}\label{e: bias.critic.grad}
    \|e_h\|_2\leq C_\phi\gamma^B+\nu_0C_\phi L_c\eta_m=:d_v.
\end{equation}
The reference estimator contributes no term to this bias because its conditional mean is the required dual derivative. The trajectory part of~\eqref{e: critic.update} has norm at most $C_\phi$, so
\begin{equation}\label{e: MSE.critic.grad}
    \mathbb E[\|\widehat v^{k,h}\|_2^2\mid\mathcal F_{k,h}]\leq 2C_\phi^2+2S_\Psi^2=:V_c.
\end{equation}
These bounds make the dependence on the dual oracle explicit.

The variational inequality at the constrained minimizer gives $\langle\nabla f_k(\zeta^*),\zeta-\zeta^*\rangle\geq0$ for $\zeta\in\Omega$. Strong convexity then implies
\[
    \langle\nabla f_k(\zeta),\zeta-\zeta^*\rangle\geq\mu\|\zeta-\zeta^*\|_2^2.
\]
Projection nonexpansiveness, expansion of the square, and Young's inequality yield
\begin{align}\label{e: critic.recursion}
    \mathbb E[\|\zeta_{h+1}^k-\zeta^*\|_2^2\mid\mathcal F_{k,h}]
    &\leq\|\zeta_h^k-\zeta^*\|_2^2
    -2\beta_{\zeta,h}\langle\zeta_h^k-\zeta^*,\nabla f_k(\zeta_h^k)+e_h\rangle
    +\beta_{\zeta,h}^2V_c
    \nonumber\\
    &\leq(1-\mu\beta_{\zeta,h})\|\zeta_h^k-\zeta^*\|_2^2
    +\beta_{\zeta,h}\frac{d_v^2}{\mu}+\beta_{\zeta,h}^2V_c.
\end{align}
This proof works at boundary minimizers and uses only bounded second moments of the stochastic update.

To solve the recursion, set $D_h=\mathbb E_k\|\zeta_h^k-\zeta^*\|_2^2$ and $E_0=d_v^2/\mu^2$. With $\beta_{\zeta,h}=2/[\mu(h+2)]$,
\[
    D_{h+1}-E_0\leq\frac{h}{h+2}(D_h-E_0)+\frac{4V_c}{\mu^2(h+2)^2}.
\]
Let $A_c=\max\{R^2,4V_c/\mu^2\}$. We claim $D_h\leq E_0+A_c/(h+1)$. It holds at $h=0$ since $\zeta_0$ is the center of $\Omega$. If it holds at $h$, then
\[
    D_{h+1}-E_0
    \leq\frac{A_ch}{(h+2)(h+1)}+\frac{A_c}{(h+2)^2}
    \leq\frac{A_c}{h+2}.
\]
This proves
\begin{equation}\label{e: critic.final}
    \mathbb E_k\|\zeta_H^k-\zeta^*\|_2^2
    \leq\frac{A_c}{H+1}+\frac{2C_\phi^2}{\mu^2}\gamma^{2B}
    +\frac{2\nu_0^2C_\phi^2L_c^2}{\mu^2}\eta_m^2.
\end{equation}
In particular, the constants include $(\mu_c\kappa)^{-2}$.

For every $x=(s,a)$,
\[
    |r(x,\zeta_H^k)-r_{\rm lin}(x,\zeta^*)|
    \leq\eta_m+C_\phi\|\zeta_H^k-\zeta^*\|_2.
\]
Square, integrate against the probability measure $\lambda_{\theta_k}$, and use~\eqref{e: critic.final}. This proves both inequalities in~\eqref{e: vanilla.bound}. \proofend

\paragraph{Proof of Lemma~\ref{thm: MLMC.critic.bound}}
The randomized-level construction follows the MLMC telescoping principle of \citet{beznosikov2023first}; we verify both its output law and cost for the coupled critic runs used here. Write $H_{\max}=2^J$, $R_H=R_{k,H}$, and $r_*^{\rm lin}=\bar r_{\theta_k}$. The common stepsize schedule ensures that, conditional on a chosen level, the prefix of length $H$ has the same law as a standalone run of length $H$.

Since $\mathbb P(N_k=j)=2^{-j}$, we have
\begin{align}\label{e: MLMC.mean}
    \mathbb E_k r_{k,\mathrm{MLMC}}
    &=\mathbb E_k R_1+\sum_{j=1}^J2^{-j}2^j\mathbb E_k(R_{2^j}-R_{2^{j-1}})
    =\mathbb E_kR_{H_{\max}}.
\end{align}
All sums are finite. Jensen's inequality in $L_2(\lambda_{\theta_k})$ and Lemma~\ref{thm: vanilla.critic.bound} give
\[
    \delta_k^2\leq\mathbb E_k\|R_{H_{\max}}-r_*^{\rm lin}\|_{L_2(\lambda_{\theta_k})}^2
    \leq C(H_{\max}^{-1}+q_m).
\]

Let $\Delta_j=R_{2^j}-R_{2^{j-1}}$. No additional cancellation assumption on coupled runs is needed for the following bound: the two marginals and the triangle-square inequality give
\begin{equation}\label{e: MLMC.increment}
    \mathbb E_k\|\Delta_j\|_{L_2(\lambda_{\theta_k})}^2
    \leq2\mathbb E_k\|R_{2^j}-r_*^{\rm lin}\|^2
    +2\mathbb E_k\|R_{2^{j-1}}-r_*^{\rm lin}\|^2
    \leq C(2^{-j}+q_m).
\end{equation}
Each norm on the middle line is in $L_2(\lambda_{\theta_k})$. Therefore
\begin{align}\label{e: MLMC.second}
    \sigma_k^2
    &\leq2\mathbb E_k\|R_1-r_*^{\rm lin}\|^2
    +2\sum_{j=1}^J2^{-j}4^j\mathbb E_k\|\Delta_j\|^2
    \nonumber\\
    &\leq C(1+q_m)+C\sum_{j=1}^J(1+2^jq_m)
    \leq C[1+J+H_{\max}q_m].
\end{align}
This accounts for both the base estimator and the statistical contribution at each level.

Finally, the number of updates is $2^{N_k}$ on $\{N_k\leq J\}$ and one otherwise. Its expectation is
\[
    \sum_{j=1}^J2^{-j}2^j+\mathbb P(N_k>J)=J+2^{-J}\leq J+1.
\]
Shared prefixes produce all three needed outputs within this budget. Each update uses one length-$B$ trajectory and one dual-oracle call. This proves the Lemma~\ref{thm: MLMC.critic.bound}. \proofend

\section{Proof of Theorem~\ref{thm: main.result}}
\label{appendix: main.result}

The proof combines the previous results while keeping the fixed finite-width error explicit. Let $T\geq2$ and use~\eqref{e: main.parameters}. The choice of $B$ implies
\[
    \gamma^B\leq T^{-2},\qquad G_1^2\gamma^B\leq\mu_F/2.
\]
Also $T\leq H_{\max}<2T$ and $e^{-\mu_F\beta_wH_w}\leq T^{-4}$. By Lemma~\ref{thm: MLMC.critic.bound},
\begin{equation}\label{e: final.bound.critic}
    \delta_k^2\leq C(T^{-1}+\eta_m^2),\qquad
    1+\sigma_k^2\leq C V_T,
    \quad V_T=1+\log(2H_{\max})+H_{\max}(\eta_m^2+\gamma^{2B}).
\end{equation}
The constants are uniform over $k$ by the assumptions.

The tail $\tau_B$ is harmless at this horizon. To see this with discount dependence displayed, use $\log(1/\gamma)\geq1-\gamma$ and the definition of $B$ to obtain
\[
    B\leq1+\frac{2\log T+\log(2G_1^2/\mu_F)}{1-\gamma}.
\]
Since $\mu_F\leq G_1^2$, the logarithm on the right is positive. Hence
$\tau_B\leq C(1+\log T)T^{-2}/(1-\gamma)$, with a fixed constant $C$. Lemma~\ref{thm: NPG.bound} now gives
\begin{align}\label{e: final.bounds.NPG}
    \|\mathbb E_kw_k-w^*(\theta_k)\|_2
    &\leq\frac{C}{1-\gamma}\big(\sqrt{\varepsilon_{\rm app}}+\eta_m+T^{-1/2}\big),
    \nonumber\\
    \mathbb E_k\|w_k\|_2^2&\leq\frac{C V_T}{(1-\gamma)^2}.
\end{align}
The small $T^{-2}\log T$ term is absorbed into $T^{-1/2}$.

Insert these bounds into Lemma~\ref{lem: opt.gap}:
\[
    \gapK\leq\sqrt{\varepsilon_{\rm bias}}
    +\frac{\Phi_0}{\alpha K}
    +\frac{C}{1-\gamma}(\sqrt{\varepsilon_{\rm app}}+\eta_m+T^{-1/2})
    +\frac{C\alpha V_T}{(1-\gamma)^2}.
\]
For $\alpha=(1-\gamma)/\sqrt{K V_T}$ and $K=T$, the potential and second-moment terms are bounded by
\[
    \frac{C}{1-\gamma}\sqrt{\frac{V_T}{T}}
    \leq\frac{C}{1-\gamma}
    \left(\sqrt{\frac{1+\log T}{T}}+\eta_m+T^{-2}\right).
\]
This proves~\eqref{e: main.result}. The stepsize adjustment is what allows a fixed representation: a nonzero $\eta_m$ appears in the approximation floor rather than as an increasing term from the critic second moment.

The NPG subroutine uses $BH_w$ environment transitions per outer iteration. Lemma~\ref{thm: MLMC.critic.bound} bounds expected critic cost by $B(1+\log_2H_{\max})$ transitions and $1+\log_2H_{\max}$ dual-oracle calls. Thus
\begin{equation}\label{e: sample.cost}
    \mathbb E N_{\rm env}\leq KB(H_w+1+\log_2H_{\max}),\qquad
    \mathbb E N_{\rm ref}\leq c_\Psi K(1+\log_2H_{\max}),
\end{equation}
where $c_\Psi$ is the reference-draw count per oracle call. For fixed discount and regularity constants, $B,H_w=\mathcal O(\log T)$, and the total is $\widetilde{\mathcal O}(T)$. Generic duals can instead report the number of oracle calls in~\eqref{e: sample.cost}. With fixed $m,R,\zeta_0$ and their conditioning constants, taking $T=\widetilde{\mathcal O}(\epsilon^{-2})$ makes the statistical term at most $\epsilon$ above the floor. A uniformly random index from $\{0,\ldots,K-1\}$ therefore returns a policy with this expected gap. This proves Theorem~\ref{thm: main.result}. \proofend

\section{Separable Utility Duals and Examples}
\label{appendix: example.separable}

This appendix develops the separable-utility specialization introduced in Section~\ref{s: dual.representation}. We first derive the general Fenchel-dual representation and associated pseudo-reward, and then specialize the construction to entropy, smoothed entropy, KL imitation, and quadratic matching. For each example, we provide the explicit dual functional, pseudo-reward, unbiased dual-gradient estimator, and the conditions that ensure bounded rewards and the dual regularity required by Assumption~\ref{a: dual.stability}. The derivations use the Fenchel-Young identity \citep{rockafellar1997convex}, as in variational RLGU \citep{zhang2020variational} and variational divergence estimation \citep{nowozin2016fgan}.

\begin{prop}[Separable dual construction]\label{prop: separable.dual}
For a separable concave utility satisfying Definition~\ref{a: separable.utility},
the variational dual and the corresponding pseudo-reward
admit the explicit forms
\begin{equation}\label{e: separable.dual}
    \Psi(r)=-C-\int f^*(-r-c)\,d\nu,
    \qquad
    r_\theta=-f'(p_\theta)-c,
\end{equation}
where
$f^*(y)=\sup_{u\in \bR}\{uy-f(u)\}$. Furthermore, whenever differentiation under the reference integral is justified,
\begin{equation}\label{e: general.dual.gradient}
 J_0^*\nabla_r\Psi(r)=\int (f^*)'(-r-c)g_0\,d\nu.
\end{equation}
\end{prop}
\noindent\emph{Proof.}
Fenchel's inequality at $p=p_\theta(x)$ gives
$f(p)+f^*(-r-c)\geq(-r-c)p$. Substituting~\eqref{e: separable.dual}, rearranging, and integrating yields
\[
    \mathbb E_{\lambda_\theta}r-\Psi(r)
    =C+\int\{pr+f^*(-r-c)\}\,d\nu
    \geq C-\int f(p)\,d\nu-\int cp\,d\nu
    =F(\lambda_\theta).
\]
Equality holds at $-r-c=f'(p)$ by Fenchel-Young equality. Hence $r_\theta=-f'(p_\theta)-c$ attains the minimum. The negative integral of a convex function is concave. Differentiating it gives~\eqref{e: general.dual.gradient}. \proofend

We next specialize Proposition~\ref{prop: separable.dual} to the utilities summarized in Table~\ref{tab: utility.examples}. The table records the resulting dual functionals, exact pseudo-rewards, and boundedness conditions; below, we derive these expressions, construct unbiased estimators of the dual-gradient term in~\eqref{e: general.dual.gradient}, and verify the curvature and moment conditions required by Assumption~\ref{a: dual.stability}.

For entropy, $f(u)=u\log u$ on $u\geq0$ (with $f(0)=0$) and $f(u)=+\infty$ for $u<0$. Maximizing $uy-u\log u$ gives $u=e^{y-1}$ and $f^*(y)=e^{y-1}$. Thus $\Psi(r)=-\int e^{-r-1}d\nu$ and $r_\theta=-1-\log p_\theta$. Positive lower and finite upper density bounds make this reward bounded. If $X\sim\nu/\nu_0$, an unbiased dual-gradient sample is
\[
    Z_\Psi(\zeta)=\nu_0e^{-r(X,\zeta)-1}g_0(X).
\]
The bounded critic range ensures that its second moment is finite. The second variation of $-\Psi$ is multiplication by $e^{-r-1}$, which is bounded above and below by positive constants on a bounded reward range.

\begin{table}[h]
\centering
\footnotesize
\setlength{\tabcolsep}{3pt}
\renewcommand{\arraystretch}{1.18}
\caption{Explicit duals. Each row uses its stated reference measure; KL imitation uses $\nu=Q$. In the smoothed row, $f_\tau(u)=u\log(u+\tau)$ for $u\geq0$ and $+\infty$ otherwise. The rightmost column ensures bounded pseudo-rewards; dual stability imposes the additional range conditions explained after the table. Full conjugate and gradient calculations appear in Proposition~\ref{prop: separable.dual}.}
\label{tab: utility.examples}
\begin{tabularx}{\textwidth}{@{}L{.21\textwidth} L{.24\textwidth} L{.275\textwidth} >{\raggedright\arraybackslash}X@{}}
\toprule
Utility & Dual functional $\Psi(r)$ & Exact pseudo-reward & Boundedness conditions \\
\midrule
Entropy\newline $-\int p\log p\,d\nu$ & $-\int e^{-r-1}\,d\nu$ & $-1-\log p$ & $0<a\leq p\leq b<\infty$ \\ \addlinespace[3pt]
Smoothed entropy\newline $-\int f_\tau(p)\,d\nu$ & $-\int f_\tau^*(-r)\,d\nu$ & $-\log(p+\tau)-p/(p+\tau)$ & $\tau>0$, $0\leq p\leq b$ \\ \addlinespace[3pt]
KL imitation\newline $-\mathrm{KL}(\lambda\|Q)$ & $-\mathbb E_Q e^{-r-1}$ & $-1-\log(d\lambda/dQ)$ & $0<a\leq d\lambda/dQ\leq b$ \\ \addlinespace[3pt]
Quadratic matching\newline $-\tfrac12\int(p-q)^2d\nu$ & $\mathbb E_Q r-\tfrac12\int r^2d\nu$ & $q-p$ & $p,q$ bounded; $dQ=q\,d\nu$ \\ \addlinespace[3pt]
\bottomrule
\end{tabularx}
\end{table}

KL imitation applies the same calculation with reference measure $Q$ and density $d\lambda_\theta/dQ$. It gives $-\mathrm{KL}(\lambda_\theta\|Q)$, the stated pseudo-reward, and the estimator $e^{-r(X,\zeta)-1}g_0(X)$ for $X\sim Q$.

For smoothed entropy, $f_\tau(u)=u\log(u+\tau)$ on $[0,\infty)$ and $+\infty$ elsewhere. Its right derivative at zero is $\log\tau$, and
\[
    f_\tau'(u)=\log(u+\tau)+\frac{u}{u+\tau},\qquad
    f_\tau''(u)=\frac1{u+\tau}+\frac{\tau}{(u+\tau)^2}>0.
\]
The conjugate maximizer is zero when $y\leq\log\tau$, and otherwise is the unique positive solution of $f_\tau'(u)=y$. Hence $(f_\tau^*)'(y)$ equals that maximizer and
\[
    Z_\Psi(\zeta)=\nu_0(f_\tau^*)'(-r(X,\zeta))g_0(X),\qquad X\sim\nu/\nu_0,
\]
is unbiased. On compact reward ranges its magnitude is bounded. Strong concavity holds when $-r$ stays in a compact interval strictly above $\log\tau$, where the conjugate curvature equals $1/f_\tau''((f_\tau^*)'(-r))$ and has positive finite bounds.

For quadratic matching, take $f(u)=u^2/2$ on all of $\mathbb R$, $c=-q$, and $C=-\tfrac12\int q^2d\nu$. Since $f^*(y)=y^2/2$,~\eqref{e: separable.dual} gives
\[
    \Psi(r)=\frac12\int q^2d\nu-\frac12\int(q-r)^2d\nu
    =\mathbb E_Qr-\frac12\int r^2d\nu,
    \qquad r_\theta=q-p_\theta.
\]
The gradient is $\mathbb E_Qg_0-\int rg_0d\nu$, and an unbiased estimator is
\[
    Z_\Psi(\zeta)=g_0(Y)-\nu_0r(X,\zeta)g_0(X),
    \qquad Y\sim Q,\quad X\sim\nu/\nu_0,
\]
with independent draws. Its second moment is bounded by the local critic and feature bounds. The dual is one strongly concave and one smooth in $L_2(\nu).$ The unrestricted quadratic extension is what yields this quadratic conjugate; a nonnegative-domain restriction would instead give a positive-part conjugate.

\section{Connection to Concave Utility RL}
\label{appendix: CURL}

\begin{prop}\label{prop: finite.curl}
For finite $\mathcal S$ and $\mathcal A$, every concave occupancy utility can be written as a concave utility of $M=|\mathcal S||\mathcal A|$ normalized discounted returns. Conversely, any concave utility of finitely many expected returns is a concave occupancy utility in finite or infinite spaces. There is a continuous-state, two-action MDP and a concave occupancy utility on it that cannot be represented by any fixed finite collection of bounded-reward returns.
\end{prop}
\noindent\emph{Proof.}
Index the finite state-action pairs by $x_1,\ldots,x_M$, and choose $r_i(x)=\mathbf1_{\{x=x_i\}}$. Their normalized returns satisfy
\[
    J_i^\pi=(1-\gamma)\mathbb E_\pi\sum_{t\geq0}\gamma^t r_i(s_t,a_t)=\lambda_\pi(x_i).
\]
Thus $F(\lambda_\pi)=f(J_1^\pi,\ldots,J_M^\pi)$ with $f=F$ on the occupancy vectors. When CURL uses unnormalized returns $\widetilde J_i^\pi$, take $f(z)=F((1-\gamma)z)$. Concavity is preserved by this linear scaling. Conversely, fixed rewards $r_1,\ldots,r_M$ define the linear map
$\lambda\mapsto(\int r_1\,d\lambda,\ldots,\int r_M\,d\lambda)$; composing it with a concave scalarization gives a concave occupancy functional. An unnormalized-return convention adds the corresponding factor $(1-\gamma)^{-1}$.

For the infinite-state statement, take $\mathcal S=[0,1]$, $\mathcal A=\{0,1\}$, $\rho$ uniform, and transitions that draw the next state uniformly, independently of the current state and action. Let $\nu$ be Lebesgue measure times counting measure on the actions. Every stationary policy has an occupancy density
\[
    p_\pi(s,1)=\pi(1\mid s),\qquad p_\pi(s,0)=1-\pi(1\mid s).
\]
Consider the concave quadratic utility
\[
    F(\lambda_\pi)=-\frac12\sum_{a=0}^1\int_0^1[p_\pi(s,a)-1/2]^2\,ds.
\]
Fix any finite collection of bounded rewards $r_1,\ldots,r_M$. Choose $M+1$ disjoint intervals of positive length. Their vectors of integrals of
$r_i(s,1)-r_i(s,0)$ in $\mathbb R^M$ are linearly dependent. Therefore a nonzero step function $h$, constant on these intervals, satisfies
\[
    \int_0^1h(s)[r_i(s,1)-r_i(s,0)]\,ds=0,\qquad i=1,\ldots,M.
\]
Scale $h$ so $\|h\|_\infty\leq1$. The policies $\pi_0(1\mid s)=1/2$ and $\pi_1(1\mid s)=1/2+h(s)/4$ have identical normalized and unnormalized returns for all $M$ rewards. However,
\[
    F(\lambda_{\pi_0})=0,\qquad
    F(\lambda_{\pi_1})=-\frac1{16}\int_0^1h(s)^2\,ds<0.
\]
Thus, no scalarization of that finite return vector represents this utility over all stationary policies. Since the finite collection was arbitrary, the statement follows. This is an objective-class separation; the convergence theorem separately imposes its policy and critic assumptions. \proofend

\section{Technical Lemmas}
\label{appendix: technical.lemmas}

The following result is adapted from the standard discounted performance-difference lemma used in policy-gradient analysis \citep{agarwal2021theory}. Concavity supplies the extension to a utility functional, as in \citet{barakat2025global}.

\begin{lemma}[Performance Difference Lemma]\label{app.lem: performance.diff}
For policies $\pi_\theta,\pi_{\theta'}$ satisfying Assumption~\ref{a: utility},
\[
    F(\lambda_\theta)-F(\lambda_{\theta'})
    \leq\mathbb E_{\lambda_\theta}A_{\theta'}^{(r_{\theta'})}.
\]
\end{lemma}
\noindent\emph{Proof.}
Concavity of $F$ gives us
\[
    F(\lambda_\theta)-F(\lambda_{\theta'})
    \leq\int r_{\theta'}\,d(\lambda_\theta-\lambda_{\theta'}).
\]
For any fixed bounded reward $r$, write $V_{\theta'}^{(r)}(s)=\int Q_{\theta'}^{(r)}(s,a)\pi_{\theta'}(da\mid s)$. The Bellman equation gives
\[
    A_{\theta'}^{(r)}(s,a)=r(s,a)+\gamma\mathbb E[V_{\theta'}^{(r)}(s')\mid s,a]-V_{\theta'}^{(r)}(s).
\]
Integrating along a trajectory of $\pi_\theta$, multiplying by $(1-\gamma)\gamma^t$, and summing makes the value terms telescope. Their tail vanishes since $V_{\theta'}^{(r)}$ is bounded. Thus
\[
    \mathbb E_{\lambda_\theta}A_{\theta'}^{(r)}
    =(1-\gamma)\big(V^{\pi_\theta}(r)-V^{\pi_{\theta'}}(r)\big)
    =\int r\,d(\lambda_\theta-\lambda_{\theta'}).
\]
Setting $r=r_{\theta'}$ then proves the claim. \proofend

The following deterministic Taylor statement is the local ingredient needed from the NTK viewpoint. It is consistent with the local-linearization approach of \citet{jacot2018neural,ganesh2025order,Satheesh2026GlobalCO}. Its hypotheses state the required uniform control explicitly.

\begin{lemma}[NTK bounds]\label{app.lem: NTK.bounds}
    Suppose $\|g_0(x)\|_2\leq C_\phi$ and $|r(x,\zeta_0)|\leq C_0$ uniformly in $x$. If
    \[
        \sup_{x,\zeta\in\Omega}\|\nabla_\zeta^2r(x,\zeta)\|_{\rm op}\leq C_H/\sqrt m,
    \]
    then, uniformly over $x$ and $\zeta\in\Omega$,
    \[
        |r-r_{\rm lin}|\leq\frac{C_HR^2}{2\sqrt m},\qquad
        \|\nabla_\zeta r-g_0\|_2\leq\frac{C_HR}{\sqrt m},\qquad
        |r(x,\zeta)|\leq C_0+C_\phi R+\frac{C_HR^2}{2\sqrt m}.
    \]
\end{lemma}
\noindent\emph{Proof.}
For $v=\zeta-\zeta_0$, Taylor's formula along the segment inside $\Omega$ gives
\[
    r(x,\zeta)-r(x,\zeta_0)-g_0(x)^\top v
    =\int_0^1(1-t)v^\top\nabla_\zeta^2r(x,\zeta_0+tv)v\,dt.
\]
Its absolute value is at most $C_H\|v\|^2/(2\sqrt m)$. Applying the fundamental theorem of calculus to the gradient gives the second bound. The final bound follows from the triangle inequality. \proofend

The Hessian condition is a quantitative hypothesis on the chosen architecture, normalization, initialization, and radius. The main theorem can instead use any certified bound $\eta_m$ in~\eqref{e: local.ntk}. All expectations in that theorem are conditional on this fixed representation; a random-initialization statement additionally requires a probability bound for the event on which these uniform inequalities hold.

The next lemma is the independent-trajectory specialization of the mean-versus-second-moment method used in \citet{ganesh2025sharper}. We give the complete proof, allowing the sampled matrix and right-hand side within an update to be dependent.

\begin{lemma}[Stochastic Linear Recursion]\label{app.lem: gen.linear}
Let
\[
    x_{h+1}=x_h-\beta(\widehat A_hx_h-\widehat b_h),\qquad x_0=0.
\]
For the history $\mathcal H_h$ before update $h$, suppose
\[
    0\preceq\widehat A_h\preceq LI_d,\quad
    \mathbb E[\widehat A_h\mid\mathcal H_h]=A\succeq aI_d,\quad
    \mathbb E[\widehat b_h\mid\mathcal H_h]=q,\quad
    \mathbb E[\|\widehat b_h\|_2^2\mid\mathcal H_h]\leq V.
\]
The constants may be conditional on a fixed auxiliary random variable, with $A$ independent of that variable. If $0<\beta\leq\min\{1/L,a/(4L^2)\}$, then
    \[
        \mathbb E x_H=[I_d-(I_d-\beta A)^H]A^{-1}q,
        \qquad
        \mathbb E\|x_H\|_2^2\leq\left(\frac{2}{a^2}+\frac{4\beta}{a}\right)V.
    \]
\end{lemma}
\noindent\emph{Proof.} Conditional expectation gives the exact mean recursion $\mathbb E x_{h+1}=(I_d-\beta A)\mathbb E x_h+\beta q.$ Summing its geometric series proves the mean identity. Since $aI_d\preceq A\preceq LI_d$ and $\beta\leq1/L$, its contraction matrix has norm at most $1-\beta a$.

For the second moment, expansion and the inequality $\|v-z\|^2\leq2\|v\|^2+2\|z\|^2$ give
\begin{align*}
    \mathbb E[\|x_{h+1}\|_2^2\mid\mathcal H_h]
    &\leq\|x_h\|_2^2-2\beta x_h^\top A x_h+2\beta x_h^\top q +2\beta^2 L^2\|x_h\|_2^2+2\beta^2V
    \\
    &\leq(1-\beta a+2\beta^2L^2)\|x_h\|_2^2
    +(\beta/a)\|q\|_2^2+2\beta^2V
    \\
    &\leq(1-\beta a/2)\|x_h\|_2^2+(\beta/a+2\beta^2)V.
\end{align*}
The middle line uses $2x_h^\top q\leq a\|x_h\|^2+\|q\|^2/a$; the final line uses $\|q\|^2\leq V$ and the stepsize condition. The contraction coefficient is nonnegative because $a\leq L$ and $\beta\leq a/(4L^2)$. Sum the scalar recursion from $x_0=0$ and use $\sum_{h\geq0}(1-\beta a/2)^h=2/(\beta a)$. This proves the claim. \proofend


The following elementary consequence of strong convexity explains the target used in Assumption~\ref{a: neural.critic.error}. Its proof is included because the dual minimizer and the best $L_2$ projection coincide only in special cases.

\begin{prop}[Approximation at the dual minimizer]\label{prop: approximation.transfer}
    Fix a policy $\theta$ and suppose $p_\theta=d\lambda_\theta/d\nu\leq D$ almost everywhere. Let
    $\mathcal E_\theta(r)=\mathbb E_{\lambda_\theta}r-\Psi(r)$ be $\mu_c$-strongly convex and $L_c$-smooth on a convex set containing $r_\theta$ and $\mathcal F_{\Omega,m}$, with $\nabla\mathcal E_\theta(r_\theta)=0$. Then its minimizer $\bar r_\theta$ over $\mathcal F_{\Omega,m}$ satisfies
    \[
        \|\bar r_\theta-r_\theta\|_{L_2(\lambda_\theta)}^2
        \leq D\frac{L_c}{\mu_c}
        \inf_{v\in\mathcal F_{\Omega,m}}\|v-r_\theta\|_{L_2(\nu)}^2.
    \]
\end{prop}
\noindent\emph{Proof.}
Strong convexity at the unrestricted stationary minimizer and smoothness give, for every $v$ in the critic class,
\[
    \frac{\mu_c}{2}\|\bar r_\theta-r_\theta\|_{L_2(\nu)}^2
    \leq\mathcal E_\theta(\bar r_\theta)-\mathcal E_\theta(r_\theta)
    \leq\mathcal E_\theta(v)-\mathcal E_\theta(r_\theta)
    \leq\frac{L_c}{2}\|v-r_\theta\|_{L_2(\nu)}^2.
\]
The density bound transfers the left-hand norm to $L_2(\lambda_\theta)$. Take the infimum over $v$. \proofend

\section{Additional Related Works}
\label{appendix: additional.related.works}

\paragraph{Variational duality and policy-gradient analysis.} Our analysis builds on the general-utility policy-gradient identity and concave-dual representation of \citet{zhang2020variational}; Section~\ref{s: dual.representation} fixes the reward convention. Lemma~\ref{app.lem: performance.diff} combines the classical performance-difference identity with the supporting hyperplane inequality for concave utilities. Appendix~\ref{appendix: opt.gap} adapts approximate NPG analysis with transferred errors \citep{agarwal2021theory,ganesh2025order}, isolating the conditional actor bias needed for the multilevel estimator.

\paragraph{Occupancy approximation in large-state RLGU.}
\citet{barakat2023reinforcement} study linear occupancy fitting, while \citet{barakat2026global} use maximum-likelihood occupancy estimation. Their representation assumptions concern distributions; ours concern the population minimizer of a dual-reward problem, with sufficient approximation conditions given in Proposition~\ref{prop: approximation.transfer}.
The improved sampling exponent in Table~\ref{tab:comparison} depends on dual stability, reward approximation, and the stated sampling access; it does not imply that either representation class uniformly dominates the other.

\paragraph{Finite-objective CURL.} \citet{agarwal2023joint,agarwal2022curl,bai2022joint} develop finite-objective CURL through model-based and policy-gradient analyses. Since CURL also denotes occupancy-functional objectives \citep{geist}, our comparison specifically concerns the fixed-finite-return formulation. \citet{ganesh2026breakingbiasbarrierconcave} apply MLMC to a finite vector of scalarization derivatives, achieving the $\epsilon^{-2}$
sampling exponent. Proposition~\ref{prop: finite.curl} connects this result to finite RLGU, with objective-count and regularity dependence inherited from the reduction. In infinite spaces, our function-valued formulation accommodates utilities beyond finite moment summaries.

\paragraph{MLMC estimation and neural actor-critic analysis.} MLMC is an established debiasing technique
\citep{blanchet2015unbiased}. Neural actor-critic analyses use local linearization and bias-sensitive NPG bounds \citep{ganesh2025order,Satheesh2026GlobalCO}, while
\citet{naskar2026hierarchical} combines multilevel estimation with neural critics for average-reward constrained problems. We specialize these tools to stochastic concave-dual optimization with reset discounted trajectories, explicitly counting dual-reference samples and retaining neural linearization error
in the MLMC second moment and actor stepsize. Our fixed-representation guarantee also tracks feature conditioning through~\eqref{e: critic.final} and~\eqref{e: gram.trace}.

\end{document}